%% file: ms.tex
\documentclass{article}
\pdfoutput=1

\usepackage[preprint,nonatbib]{neurips_2020}

\usepackage[utf8]{inputenc} 
\usepackage[T1]{fontenc}    
\usepackage{hyperref}       
\usepackage{url}            
\usepackage{booktabs}       
\usepackage{amsfonts}       
\usepackage{nicefrac}       
\usepackage{microtype}      

\usepackage{amsmath}
\usepackage{subfig}
\usepackage{xspace}
\usepackage{graphicx}
\usepackage{algorithm}
\usepackage{algorithmic}
\usepackage{varwidth}
\usepackage{balance}
\usepackage{bm}

\newtheorem{theorem}{Theorem}
\newtheorem{proposition}{Proposition}
\newtheorem{lemma}{Lemma}
\newtheorem{proof}{Proof}
\newtheorem{definition}{Definition}
\newcommand{\method}{\textsc{PriVAE}\xspace}

\title{Differentially Private Variational Autoencoders with Term-wise Gradient Aggregation}

\author{%
  Tsubasa Takahashi \\
  LINE Corporation \\
  \texttt{tsubasa.takahashi@linecorp.com} \\
  \And
  Shun Takagi \thanks{Work done at LINE Corporation.}\\
  Kyoto University \\
  \texttt{takagi.shun.45a@st.kyoto-u.ac.jp} \\
  \AND
  Hajime Ono \footnotemark[1]\\
  SOKENDAI \\
  \texttt{hono@ism.ac.jp} \\
  \And
  Tatsuya Komatsu \\
  LINE Corporation \\
  \texttt{komatsu.tatsuya@linecorp.com} \\
}

\begin{document}

\maketitle

\begin{abstract}
This paper studies how to learn variational autoencoders with a variety of divergences under differential privacy constraints.
We often build a VAE with an appropriate prior distribution to describe the desired properties of the learned representations and introduce a divergence as a regularization term to close the representations to the prior.
Using differentially private SGD (DP-SGD), which randomizes a stochastic gradient by injecting a dedicated noise designed according to the gradient's sensitivity, we can easily build a differentially private model.
However, we reveal that attaching several divergences increase the sensitivity from $O(1)$ to $O(B)$ in terms of batch size $B$.
That results in injecting a vast amount of noise that makes it hard to learn.
To solve the above issue, we propose term-wise DP-SGD that crafts randomized gradients in two different ways tailored to the compositions of the loss terms.
The term-wise DP-SGD keeps the sensitivity at $O(1)$ even when attaching the divergence.
We can therefore reduce the amount of noise.
In our experiments, we demonstrate that our method works well with two pairs of the prior distribution and the divergence.
\end{abstract}

\input{010_intro.tex}

\input{020_preliminary.tex}
\input{030_analysis}

\input{040_proposal.tex}
\input{050_evaluation.tex}

\input{060_conclusion.tex}

\bibliography{ref}
\bibliographystyle{abbrv}

\end{document}

%% file: 010_intro.tex
\section{Introduction}

Privacy-preserving data synthesis (PPDS) is a solution to sharing private data by constructing a generative model while preserving privacy.
Differential privacy (DP) \cite{dwork2006differential} is a rigorous notation of privacy to release statistics, and is used in broad domains and applications \cite{abowd2018us,amin2019differentially,bindschaedler2016synthesizing,chaudhuri2019capacity,papernot2016semi,schein2019locally,wang2019sparse}.
In recent years, several works have proposed differentially private deep generative models \cite{acs2018differentially, jordon2018pate, torkzadehmahani2019dp, xie2018differentially}.

Deep generative models have significantly improved in the past few years.
Variational autoencoder (VAE) \cite{kingma2016improved, kingma2013auto} is a likelihood-based model to reconstruct training inputs. 
VAE also enables us to generate random samples from its learned representations.
We often build a VAE with an appropriate prior distribution to describe the desired properties of the representations (such as encouraging clustering, sparsity, and disentanglement), and introduce a divergence as a regularization term to close the learned representations to the prior \cite{alemi2016deep,bengio2013representation,eastwood2018framework,esmaeili2019structured,mathieu2019disentangling}.
This paper studies how to learn variational autoencoders with a variety of divergences under differential privacy constraints.

A simple way to build a differentially private VAE is to employ differentially private stochastic gradient descent (DP-SGD) \cite{abadi2016deep} in the learning process of vanilla VAE.
The key idea of DP-SGD is that it injects noises to stochastic gradients for giving DP guarantees to the learned parameters.
The noise scale is designed according to the stochastic gradient's sensitivity, which is the maximal change of the gradient when any single input is modified.
To limit the gradient's sensitivity, DP-SGD first decomposes input samples into disjoint smaller groups (i.e., micro-batches).
Then, DP-SGD computes a stochastic gradient for each group and clips the norm of the gradient by a constant.
On the other hand, a misuse of the gradient aggregations might cause privacy leakages unconsciously.

Our contributions are three-fold. First, we reveal that several divergences might increase the stochastic gradients' sensitivity when attaching them to the loss function.
To discover the issues, we address a sensitivity study in the learning process of VAEs based on DP-SGD.
The sensitivity is increased from $O(1)$ to $O(B)$ in terms of batch size $B$ by attaching the divergence.
Consequently, the sensitivity increase degrades the quality of the learned model since it directly amplifies the amount of noise.
If unfortunately, we do not notice the sensitivity increase, we might cause an insufficient differential privacy guarantee.

Second, to solve the above issue, we propose \textit{term-wise DP-SGD} that crafts randomized gradients in two different ways tailored to the compositions of the loss terms.
The term-wise DP-SGD keeps the sensitivity at $O(1)$ even when attaching the divergence.
We can therefore build a differentially private VAE with a small amount of noise by our proposed method.

Third, based on the term-wise DP-SGD, we present \method, a general model to learn VAEs with attaching a variety of divergences while satisfying differential privacy.
Our experiments demonstrate that our proposed method works well with two pairs of the prior distribution and the divergence.

This paper clarifies how to aggregate gradients in VAEs to satisfy differential privacy while refraining the amount of noise.
Although we mainly study differentially private VAEs, these contributions also have significant importance for the other machine learning models to satisfy differential privacy.

\subsection{Related Works}

Generative models under differential privacy have been studied in a last decade.
Traditional approaches are based on capturing probabilistic models, low rank structure, and learning statistical characteristics from original sensitive database \cite{chen2015differentially,zhang2014privbayes,zhang2016privtree}.
Plausible deniability \cite{bindschaedler2017plausible} is an extended privacy metric behind DP for building a generative model.

We have several studies about DP-SGD \cite{mcmahan2017learning, mcmahan2018general, yu2019differentially}.
Lee et al. \cite{lee2018concentrated} demonstrated that DP-SGD can be improved with adaptive step sizes and careful allocation of privacy budgets between iterations.
Bagdasaryan et al. \cite{bagdasaryan2019differential} revealed that if the original model is unfair, the unfairness becomes worse once DP is applied.

%% file: 020_preliminary.tex
\section{Preliminaries}\label{sec:prelim}

\subsection{Differential Privacy}

Differential privacy \cite{dwork2006differential, dwork2011differential, dwork2011firm} is a rigorous mathematical privacy definition, which quantitatively evaluates the degree of privacy protection when we publish statistical outputs.
The definition of differential privacy is as follows:

\begin{definition}[($\varepsilon, \delta$)-differential privacy]
A randomized mechanism $\mathcal{M}:\mathcal{D}\rightarrow\mathcal{Z}$ satisfies ($\varepsilon, \delta$)-differential privacy if, for any two neighboring input $D, D' \in \mathcal{D}$ and any subset of outputs $Z \subseteq \mathcal{Z}$, it holds that
\begin{equation}
  \Pr[\mathcal{M}(D)\in Z] \leq \exp(\varepsilon) \Pr[\mathcal{M}(D')\in Z] + \delta .
\end{equation}
\end{definition}

Practically, we employ a randomized mechanism $\mathcal{M}$ that ensures differential privacy for a function $f$.
The mechanism $\mathcal{M}$ perturbs the output of $f$ to cover $f$'s sensitivity that is the maximum degree of change over any pairs of $D$ and $D'$.

\begin{definition}[Sensitivity]
The sensitivity of $f$ for any two neighboring input $D, D' \in \mathcal{D}$ is 
\begin{equation}
    \Delta_{f} = \sup_{D, D' \in \mathcal{D}} ||f(D)-f(D')||.
\end{equation}
where $||\cdot||$ is a norm function defined on $f$'s output domain.
\end{definition}
Based on the sensitivity of $f$, we design the degree of noise to ensure differential privacy.
Laplace mechanism and Gaussian mechanism are well-known as standard approaches.

Let $M_1, \dots, M_k$ be mechanisms satisfying $(\varepsilon_1, \delta_1)$-, $\dots$, $(\varepsilon_k, \delta_k)$-differential privacy, respectively.
Then, a mechanism sequentially applying $M_1, \dots, M_k$ satisfies ($\sum_{i\in[k]} \varepsilon_i$, $\sum_{i\in[k]} \delta_i$)-differential privacy.
This fact refers to \textit{composability}\cite{dwork2006differential}.
In particular, this composition is called sequential composition.

\subsection{DP-SGD}

Differentially private stochastic gradient descent (DP-SGD) \cite{abadi2016deep}, is a useful optimization technique for learning a model $f$ under differential privacy constraints.
The key idea of DP-SGD is that it adds noise to stochastic gradients during training for making differential privacy guarantees on $f$'s parameters $\theta$.
To obtain the scale of noise, DP-SGD limits $\ell_2$-sensitivity of stochastic gradient $\mathbf{g}$ by clipping its norm.
The gradient clipping $\pi_C$ that limits the sensitivity up to $C$ is denoted as follows:
\begin{equation}
    \pi_C(\mathbf{g}) = \mathbf{g} * \min \left(1,\frac{C}{||\mathbf{g}||_2} \right)
    \label{eq:clip}
\end{equation}
In the DP-SGD, we compute an empirical loss for each micro-batch that includes only one sample.
For each micro-batch, DP-SGD generates its clipped gradient. 
Based on the clipped gradients, DP-SGD crafts a randomized gradient $\Tilde{\mathbf{g}}$ through computing the average over the clipped gradients and adding noise whose scale is defined by $C$ and $\sigma_\varepsilon$, where $\sigma_\varepsilon$ is noise scaler to satisfy $(\varepsilon,\delta)$-DP.
\begin{equation}
    \Tilde{\mathbf{g}} = 
        \frac{1}{B} \left(\sum_{i\in[B]} \pi_C(\mathbf{g}_i) + \mathcal{N}(0,(\sigma_{\varepsilon} C)^2\textbf{I})\right).
    \label{eq:grad_termwise_combine}
\end{equation}
At last, DP-SGD takes a step based on the randomized gradient $\Tilde{\mathbf{g}}$.
Abadi et al. \cite{abadi2016deep} also proposed a moment accountant that maintain privacy loss more precisely than the sequential composition.
In the moment accountant, $\sigma_\varepsilon$ has the following relationship against $\varepsilon$ and $\delta$ (Theorem 1 in \cite{abadi2016deep}).
\begin{equation}
    \sigma_\varepsilon \geq c_2 \frac{q\sqrt{T\log (1/\delta)}}{\varepsilon}
    \label{eq:sigma_eps}
\end{equation}
where $q$ is a sampling probability, $T$ is a number of steps and $c_2$ is a constant number.
To compute the privacy loss through moment accountant, we can utilize a useful tool in Tensorflow privacy \cite{tensorflowprivacy}.

\subsection{Variational Autoencoder}

Variational autoencoder (VAE) \cite{kingma2013auto} is a model to learn parametric latent variables by maximizing the marginal log-likelihood of the training data points.
VAE consists of two parts, inference model $q(z|x)$ for an encoder $g(x;\theta)$, and the likelihood model $p(x|z)$ for a decoder $f(z;\theta)$.

\paragraph{Variational evidence lower bound.}
Introduction of an approximate posterior $q_{\phi}(z|x)$ enable us to construct variational evidence lower bound (ELBO) on log-likelihood $\log p(x)$ as
\begin{equation}
\begin{split}
    \mathcal{L}_{ELBO} &= \log p(x) - D_{KL}(q(z|x) || p(x|z)) \\
    &= \mathbb{E}_{q(z|x)}[\log p(x|z)] - D_{KL}(q(z|x) || p(z)) \\
    &\leq \log p (x) . \\
\end{split}
\end{equation}

To implement encoder and decoder as a neural network, we need to backpropagate through random sampling.
However, such backpropagation does not flow through the random samples.
To overcome this issue, VAE introduces the reparametrization trick.
The trick can be described as $z = \mu + \Sigma \epsilon$ where $\epsilon \sim \mathcal{N}(0, \bm{I})$.
After constructing VAE, we can generate random samples following the two steps; 1) choose a latent vector $z \sim \mathcal{N}(0, \sigma^2I)$, and 2) generate $\Tilde{x}$ by decoding $z$. $\Tilde{x}$ = $f(z)$.

\paragraph{Attaching a divergence for regularization.}
To capture the desired property in the learned representation space of VAEs, we can employ a variety of prior distributions as $p(z)$ and an additional regularization term.
We assume an additional regularization term $D(q(z), p(z))$, that is a divergence between $q(z)$ and $p(z)$.
The ELBO with the regularization is described as follows:
\begin{equation}
    \mathcal{L}_{\text{ELBO}} 
    = \mathbb{E}_{q(z|x)}[\log p(x|z)] 
    - \beta D_{\text{K}L}(q(z|x) || p(z)) 
    - \alpha D(q(z), p(z))
    \label{eq:elbo_reg}
\end{equation}
Several $D(q(z), p(z))$ are difficult to be decomposed into micro-batch losses that DP-SGD requires.

%% file: 030_analysis.tex
\section{Sensitivity Analysis} \label{sec:analysis}

Here we address a sensitivity study in DP-SGD for VAEs with various loss functions to clarify the required nose scale for ensuring differential privacy on the parameters of VAEs.

\subsection{Learning VAEs in DP-SGD}

Let $batch=\{x_i\}_{i=1}^B$ is a randomly selected samples with sampling probability $B/N$.
We assume the loss function of VAE is formed as the following abstract equation:
\begin{equation}
    \mathcal{L} = - \mathcal{L}_{ELBO} = \mathbb{E}_{x} \left[ \phi(x) \right] + \psi(batch)
    \label{eq:lfb3a}
\end{equation}
where  $\phi(x_i)$ is a function which computes a loss only depend on $x_i$, and $\psi(batch)$ is a function which computes loss value across all samples in \textit{batch} (=$\{x_1,\dots, x_B\}$).
We call $\phi(x_i)$ \textit{sample-wise term}, and $\psi(batch)$ \textit{batch-wise term}.
The loss function (\ref{eq:lfb3a}) is also rewritten as follows: 
\begin{equation}
\textstyle
    \mathcal{L} = \frac{1}{B} \sum_{i \in [B]} \mathcal{L}_i 
                = \frac{1}{B} \sum_{i \in [B]} \left(\phi(x_i) + \psi(batch) \right)
    \label{eq:lfb3b}
\end{equation}
where $\mathcal{L}_i$ is a micro-batch loss.
In DP-SGD, the stochastic gradient of $\mathcal{L}_i$ is clipped by $C$ as (\ref{eq:clip}).
That means the sensitivity of the gradient is bounded by the constant.
At the last step in a batch, we craft a randomized gradient through aggregating the clipped gradients and injecting noise whose scale is $C\sigma_{\varepsilon}$ to ensure differential privacy.
This aggregation has an effort to reduce the variance of the noise.
We call the construction of (\ref{eq:lfb3b}) \textit{micro aggregation}.

Based on the above assumptions, we can see the following series of propositions.
\begin{proposition}\label{pf:1}
Assume $\mathcal{L}_i = \phi(x_i)$ and the stochastic gradient of $\mathcal{L}_i$ is clipped by (\ref{eq:clip}) with the constant $C$, $\ell_2$-sensitivity of $\mathcal{L}$ is $C$.
\end{proposition}
\begin{proof}
Let $\textbf{g}_i$ be the stochatic gradient of $\mathcal{L}_i = \phi(x_i)$.
Since $\phi(x_i)$ is independent from $\phi(x_j)$ of $j \neq i$, changing $x_i$ only modifies its clipped gradient $\pi_C(\textbf{g}_i)$.
Thus, the sensitivity is $1 \times C$.
\end{proof}

\begin{proposition}\label{pf:2}
Assume $\mathcal{L}_i = \psi(batch)$ and the stochastic gradient of $\mathcal{L}_i$ is clipped by (\ref{eq:clip}) with the constant $C$, $\ell_2$-sensitivity of $\mathcal{L}$ is $BC$.
\end{proposition}
\begin{proof} 
Let $\textbf{g}_i$ be the stochatic gradient of $\mathcal{L}_i = \psi(batch)$.
While $\psi(batch)$ is shared in all $\mathcal{L}_j$, $\forall j\in[B]$, the change of $x_i$  modifies all $\mathcal{L}_j$.
Thus, the sensitivity is $B \times C$.
\end{proof}

\begin{proposition}\label{pf:3}
Assume $\mathcal{L}_i = \phi(x_i) + \psi(batch)$ and the stochastic gradient of $\mathcal{L}_i$ is clipped by (\ref{eq:clip}) with the constant $C$, $\ell_2$-sensitivity of $\mathcal{L}$ is $BC$.
\end{proposition}
\begin{proof}
As well as the Proof of Proposition \ref{pf:2}, since $\psi(batch)$ is shared in all $\mathcal{L}_j$, $\forall j\in[B]$, the change of $x_i$ modifies all $\mathcal{L}_j$. 
Thus, the sensitivity is $B \times C$.
\end{proof}

From the above three propositions, we reach the following theorem about the sensitivity for learning differentially private VAEs in the DP-SGD manner.

\begin{theorem}\label{theo:sensitivity}
$\ell_2$-sensitivity of $\mathcal{L}$ for learning a vanilla VAE is either $BC$ or $C$.
\end{theorem}
\begin{proof}
Let $rec_i$ be the reconstruction loss (i.e., negative log-likelihood) of $x_i$.
The loss functions of vanilla VAE can be written as
$\mathcal{L}_i = rec_i + \text{KL}(q(z|x)||p(z))$.
For this formulation, the sensitivity is $BC$ from the proposition \ref{pf:3}.
Fortunately, the KL term can be decomposed as follows:
\begin{equation}
    \text{KL}(q(z|x)||p(z)) = \frac{1}{B}\sum_{i \in [B]} (\log q(z|x_i) - \log p(z))
\end{equation}
Thus we can rewrite the loss as sample-wise form that does not depend on the other samples:
\begin{equation}
    \mathcal{L}_i = \text{rec}_i + \text{kld}_i \quad  \text{where} \quad \text{kld}_i=\log q(z|x_i) - \log p(z)
    \label{eq:vae1}
\end{equation}
From Proposition \ref{pf:1}, the sensitivity when we utilize (\ref{eq:vae1}) is $C$.
\end{proof}

\begin{lemma}
Let a VAE introduces an additional regularization term and the regularization term cannot be decomposed into micro-batch losses that every micro-batch depends on an only single input.
The sensitivity of DP-SGD for learning the VAE with the regularization is $BC$.
\end{lemma}

On the other hand, DP-SGD is applicable not only for micro-batches but also for the overall batch ($\mathcal{L}=\psi(batch)$).
When we craft the randomized gradient from the overall batch, the stochastic gradient' sensitivity keeps at $O(1)$.
We call this construction \textit{batch aggregation}.
By employing the batch aggregation, we can compute a divergence from all samples in the batch without increasing the sensitivity.
However, the batch aggregation also injects a large amount of noise because it does not have a factor to reduce the noise that micro-batch organizations have.
Therefore, DP-SGD often organizes the micro-batches whose size is one for crafting the randomized gradient.

\subsection{Privacy Leakage}

As discussed the above study, ill constructions of the randomized gradient that aggregates micro-batch losses like $\mathcal{L}_i = \text{rec}_i +\text{kld}_i + D(q(z), p(z))$ and injects insufficient scale of noise to cover the increased sensitivity fail into differential privacy guarantee that we expected.
In this case, unfortunately, the information of $D(q(z), p(z))$ that depends on inputs of the whole batch is leaked.
By this leaked sensitive information, we might get beautiful results, but it is the result of our poor understanding of gradient constructions in the DP-SGD manner.

\subsection{Augmentation for Estimating Reconstruction Error} \label{sec:marginal}

Back to the original VAE \cite{kingma2013auto}, the stochastic gradient variational Bayes (SGVB) estimator enables us to compute the ELBO over a single batch as:
\begin{equation}
    \mathcal{L}_{ELBO}(x_i) = -\mathcal{L}_i = -\text{kld}_i + \frac{1}{L}\sum_{l=1}^{L}\log p(x_i|z_{i,l})
\label{eq:elbo_L}
\end{equation}
In the original VAE, we can set $L=1$ if the batch size is large enough \footnote{\cite{kingma2013auto} mentioned that L can be set to 1 as long as the minibatch size was large enough. e.g. $B=100$.}.
However, DP-SGD assumes \textit{micro-batches} whose size is 1. 
In order to accurately estimate the log-likelihood around $x_i$, we should set $L$ in no small number.
Thanks to gradient clipping (\ref{eq:clip}), the sensitivity is still bounded by $C$ even when utilizing a large $L$.
Since $\frac{1}{L}\sum_{l \in [L]}\log p(x_i|z_{i,l})$ is independent from $x_{j \neq i}$, and the stochastic gradient including it is clipped by the constant $C$, the sensitivity is bounded by $C$ against any $L$.

From the above discussion, we can utilize augmentations that reduce the reconstruction error without increasing the sensitivity.
However, it consumes much more computational time and memory spaces.

%% file: 040_proposal.tex
\section{Proposed Method}

Based on the sensitivity analysis, we present how to learn differentially private variational autoencoders with suppressing the amount of noise.
We first introduce a general model \method that learns variational autoencoder in a differentially private way.
Second, we propose a novel learning technique \textit{term-wise DP-SGD} that reduces the amount of noise for DP by decomposing stochastic gradients into term-wise components.
Our proposed method also utilizes the augmentation that attempts to reduce the reconstruction error, as discussed in section \ref{sec:marginal}.

\subsection{\method: a general model of differentially private VAE}

Our basic idea is to decompose the terms of the loss function into two groups and compose a noisy gradient that ensures the DP group by group.
For each group, we separately run the gradient aggregation sequence for DP, which consists of computing stochastic gradients, clipping gradients, and adding noise as following the DP-SGD manner.

Towards reducing the amount of noise, we first introduce the notation of partitions.
Let $s$ be a partition of \textit{batch}, where $s=\{x_1, \dots, x_{|s|}\}$, $x_i \in  batch$.
Any pairs of $s_j$ and $s_k$ ($j \neq k$) are mutually disjoint, that is $s_j\cap s_{k\neq j}=\emptyset$.

\paragraph{Objective function of \method .}
\method minimizes the objective function described below:
\begin{equation}
\textstyle
    \mathcal{L} = \mathcal{L}_{sample} + \mathcal{L}_{batch} = \mathbb{E}_{x} \left[ \phi(x) \right] + \mathbb{E}_{s} \left[ \psi(s) \right]
    = \frac{1}{B}\sum_{i \in [B]} \phi(x_i) + \frac{1}{b}\sum_{j \in [b]} \psi(s_j)
    \label{eq:privae_loss}
\end{equation}
where $\phi(x)=\text{rec}(x) + \beta \text{kld}(x)$, $\psi(s)=\alpha D(q(z), p(z))$ and $s$ is the partition denoted the above.
Let $b$ is the number of the partitions, and $\bigcup_{j \in [b]} s_j = batch$.
Note, (\ref{eq:privae_loss}) with $b=1$ is identical to (\ref{eq:lfb3a}).

\input{041_alg2.tex}

\subsection{Termwise DP-SGD}

We propose termwise DP-SGD that composes noisy gradient for DP in a term-wise way.
The termwise DP-SGD crafts the noisy gradients for sample-wise terms $\phi(x_i)$ and batch-wise terms $\psi(s_j)$, separately.
In the last phase of termwise DP-SGD, it combines these noisy gradients and updates parameters $\theta$.
The overall proposed procedure of termwise DP-SGD is in Algorithm \ref{alg2}.

\paragraph{Gradient aggregation for sample-wise term.}
For each sample-wise term $\phi(x_i)$, we craft its clipped gradient $\pi_{C_1}(\nabla_{\theta} \phi(x_i; \theta))$ with clip size $C_1$.
We then aggregate the sum of the clipped gradients as follows:
\begin{equation}
    \textstyle
    \Bar{\mathbf{g}}_{\text{sample}} = \sum_{i \in [B]} \pi_{C_1}(\nabla_{\theta} \phi(x_i; \theta)) .
    \label{eq:grad_sample}
\end{equation}

\paragraph{Gradient aggregation for batch-wise term.}
For the batch-wise terms $\psi(s_j)$, we first partition $batch$ into sub-groups $s_1, \dots, s_b$ where $b \leq B$.
We then compute $\psi(s_j)$ for $j \in [b]$ and aggregate their clipped gradients with clip size $C_2$ as described below:
\begin{equation}
    \textstyle
    \Bar{\mathbf{g}}_{\text{batch}} = \sum_{j \in [b]} \pi_{C_2}(\nabla_{\theta} \psi(s_j; \theta)) .
    \label{eq:grad_batch}
\end{equation}

\paragraph{Term-wise noise injections and concatenation.}
Finally we combine the above two gradients as 
\begin{equation}
    \Tilde{\mathbf{g}} = 
        \frac{1}{B} \left(\Bar{\mathbf{g}}_{\text{sample}} + \mathcal{N}(0,(\sigma_{\varepsilon}'C_1)^2\textbf{I})\right)
        + \frac{1}{b} \left(\Bar{\mathbf{g}}_{\text{batch}} + \mathcal{N}(0,(\sigma_{\varepsilon}'C_2)^2\textbf{I})\right).
    \label{eq:grad_termwise_combine}
\end{equation}
where $\sigma_{\varepsilon}' = \kappa \sigma_{\varepsilon}$, $\kappa=\frac{\sigma_{\varepsilon/2}}{\sigma_{\varepsilon}}=2\sqrt{\frac{\log \delta/2}{\log \delta}}$ if $C_2$ > 0, otherwise $\kappa=1$.
$\kappa$ can be derived from (\ref{eq:sigma_eps}).
The ratio of $C_2$ to $C_1$ plays an important role that adjusts the scale of (\ref{eq:grad_batch}) against (\ref{eq:grad_sample}) like $\alpha$.

\subsection{Discussion}

We here discuss the privacy guarantee and noise scale of our proposed method.

\begin{theorem}
Term-wise DP-SGD with the noise scale $\sigma_{\varepsilon}'$ satisfies ($\varepsilon$, $\delta$)-differential privacy if DP-SGD with $\sigma_{\varepsilon}$ satisfies ($\varepsilon$, $\delta$)-differential privacy for a VAE that has no batch-wise terms. 
\label{theo:proposed_dp}
\end{theorem}
\begin{proof}
$\sigma_{\varepsilon}' = \kappa \sigma_{\varepsilon}$ is the noise scale that satisfies ($\varepsilon/2$, $\delta/2$)-DP.
From the sequential composition of the first term and the second term in (\ref{eq:grad_termwise_combine}), the sum of the two terms satisfies ($\varepsilon$, $\delta$)-DP.
\end{proof}

\begin{lemma}
$\ell_2$-sensitivity of $\mathcal{L}$ (\ref{eq:privae_loss}) is $C_1 + C_2$. That means the sensitivity is $O(1)$.
\end{lemma}

\begin{proof}
Since all $s_j$ and $s_{k\neq j}$ are disjoint, the change of any single $x_i \in batch$ influences only $\phi(x_i)$ and $\psi(s_j)$ where $x_i \in s_j$.
Thus, $\ell_2$-sensitivity of $\mathcal{L}_{sample}$ and $\mathcal{L}_{batch}$ is $C_1$ and $C_2$, respectively.
Finally, $\ell_2$-sensitivity of $\mathcal{L}$ is $C_1+C_2$.
\end{proof}

In (\ref{eq:privae_loss}) and (\ref{eq:grad_termwise_combine}), the computation of $\psi$ for each partition results in under-estimation against $\psi(batch)$, but it brings the reduction of the noise variance for the second term.
In (\ref{eq:grad_termwise_combine}), the noise $\mathcal{N}(0,(\sigma_{\varepsilon}'C_2)^2\textbf{I})$ can be divided by the number of partitions $b$.
Therefore, we can manipulate the degree of the trade-off between the estimation accuracy of $\psi(batch)$ and the second term's noise scale by $b$.

Finally, we discuss the noise scale.
In the existing method DP-SGD with a divergence, the overall noise scale is $\Sigma_{\varepsilon}=\frac{BC\sigma_{\varepsilon}}{B}=C\sigma_{\varepsilon}$.
While our term-wise DP-SGD has $\Sigma_{\varepsilon}'=\frac{C_1\sigma_{\varepsilon}'}{B}+\frac{C_2\sigma_{\varepsilon}'}{b} \approx 2\left(\frac{C_1}{B}+\frac{C_2}{b}\right)\sigma_{\varepsilon}$ by using $\sigma_{\varepsilon}' =\kappa \sigma_{\varepsilon} \approx 2\sigma_{\varepsilon}$.
In the DP-SGD with divergence, the order of the noise scale can be written as $O(1)$, while our proposed method has $O(1/b)$ since $B \geq b$.

Table \ref{tb:comp} summarizes the sensitivity and noise scale of DP-SGD and our term-wise DP-SGD.

\begin{table}[t!]
    \centering
    \caption{Sensitivity and noise scale for learning differentially private models with a batch-wise term.}
    \begin{tabular}{lccc}
        \toprule
        & $\mathcal{L}$ & $\ell_2$-sensitivity of $\mathcal{L}$ & noise scale \\ 
        \midrule
        DP-SGD (micro agg.) 
        & $\frac{1}{B}\sum_{i=1}^{B} (\phi(x_i) + \psi(batch)) $
        & $BC$ & $C\sigma_\varepsilon$ \\
        DP-SGD (batch agg.) 
        & $\psi(batch)$
        & $C$ &  $C\sigma_\varepsilon$ \\
        \textbf{Term-wise DP-SGD}
        & $\frac{1}{B}\sum_{i=1}^{B} \phi(x_i) + \frac{1}{b}\sum_{j=1}^{b} \psi(s_j)$
        & $C_1+C_2$ & $(\frac{C_1}{B}+\frac{C_2}{b})\kappa\sigma_\varepsilon$ \\
        \bottomrule
    \end{tabular}
    \label{tb:comp}
\end{table}

%% file: 041_alg2.tex
\begin{algorithm}[tb]
    \caption{Termwise DP-SGD}
    \label{alg2}
    \begin{algorithmic}
        \STATE {\bfseries Input:} $x_1, \dots, x_N$
        \STATE {\bfseries Parameters:} learning rate $\eta_t$, noise scale $\sigma_{\varepsilon}'$, batch size $B$, \#partitions  $b$, clipping size $C_1$ and $C_2$
        \STATE {Initialize $\theta_0$ randomly\;}
        \FOR{$t$ {\bfseries in} $[T]$}
            \STATE randomly sample \textit{batch} with probability $B/N$
            \STATE $\Bar{\mathbf{g}}_{\text{sample}} \leftarrow \sum_{i \in [B]} \pi_{C_1}(\nabla_{\theta_t} \phi(x_i; \theta_t))$
            \STATE randomly sample sub-group $s_j$ from \textit{batch} with probability $b/B$
            \STATE $\Bar{\mathbf{g}}_{\text{batch}} \leftarrow \sum_{j \in [b]} \pi_{C_2}(\nabla_{\theta_t} \psi(s_j; \theta_t))$
            \STATE $\Tilde{\mathbf{g}} \leftarrow \frac{1}{B}(\Bar{\mathbf{g}}_{\text{sample}} + \mathcal{N}(0,(\sigma_{\varepsilon}'C_1)^2\textbf{I}))
            + \frac{1}{b}(\Bar{\mathbf{g}}_{\text{batch}} + \mathcal{N}(0,(\sigma_{\varepsilon}'C_2)^2\textbf{I}))$ 
            \STATE $\theta_{t+1} \leftarrow \theta_t - \eta_t \Tilde{\mathbf{g}}$
        \ENDFOR
        \STATE {\bfseries Output:} $\theta_T$
    \end{algorithmic}
\end{algorithm}

%% file: 050_evaluation.tex
\section{Evaluation}

In this section, we demonstrate the effectiveness of our proposed method \method with two different tasks.
We evaluate our method in a sparse coding task and a clustering task.
Each task employs a different prior distribution as $p(z)$ and divergence as the regularization term $\psi(s)$.
The experimental settings, including datasets, neural network architectures, construction of prior distributions, regularization divergences, and evaluation metrics, follow the experiments in \cite{mathieu2019disentangling}.
The experimental codes are developed in Python 3.7 and PyTorch 1.5 \cite{paszke2017automatic} and run on machines with a Tesla V100 GPU.

\input{052_sparsity}

\input{051_clustering}

%% file: 052_sparsity.tex
\subsection{Sparsity}

\begin{figure}[t]
    \centering    
    \subfloat[Sparsity]{
        \includegraphics[width=0.285\hsize]{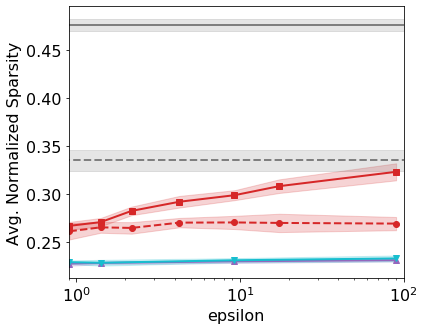}
        \label{fig:fm_sparsity}
    }
    \subfloat[Log-likelihood]{
        \includegraphics[width=0.3\hsize]{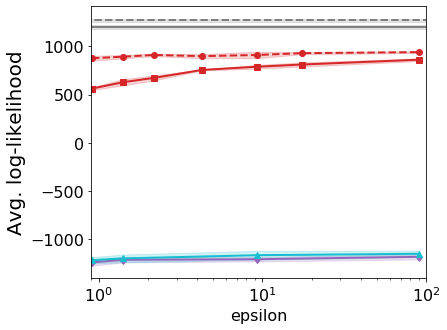}
        \label{fig:fm_recon}
    }
    \subfloat[MMD($q_\phi(z)$, $p(z)$)]{
        \includegraphics[width=0.39\hsize]{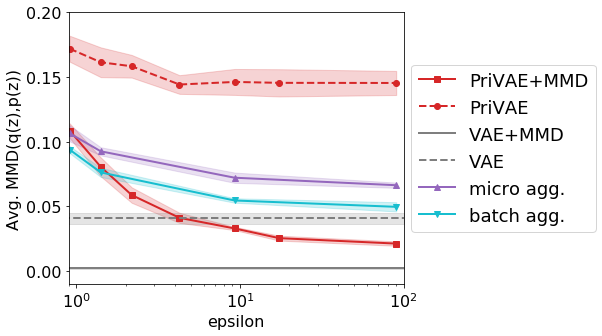}
        \label{fig:fm_mmd}
    }
    \caption{Our model \method works with MMD as a regularization term, and obtains sparsity.}
    \label{fig:sparse_prior}
\end{figure}

We first consider a sparse representation that only a small fraction of available factors are employed for reconstructions.
In this task, we utilize the Fashion-MNIST dataset \cite{xiao2017fashion}.
As well as \cite{mathieu2019disentangling}, we construct a sparse prior as $p(z)=\Pi_d(1-\gamma)\mathcal{N}(z_d;0,1)+\gamma \mathcal{N}(z_d;0,\sigma^2_0)$ with $\sigma^2_0=0.05$.
This mixture distribution can be interpreted as a mixture of samples being either \textit{off} or \textit{on}, whose proportion is set by $\gamma$.
We set $\gamma$=0.8.
The regularization term we utilize here is a dimension-wise MMD with a sum of Cauchy kernels on each dimension ($k(\bm{x},\bm{y})=\sum^D_{d=1}\sum^L_\ell \frac{\sigma_\ell}{\sigma_{\ell=1}+(x_d-y_d)^2}$) with $\sigma_\ell \in \{0.2,0.4,1,2,4,10\}$.
To measure a sparsity of the latent representations, we employ the sparsity metric defined with the Hoyer extrinsic metric \cite{hurley2009comparing} as follows:
\begin{equation}
\begin{gathered}
\textstyle
    \text{Sparsity} = \frac{1}{n}\sum_{i \in [n]} \text{Hoyer}(\Bar{\bm{z}}_i), \quad
    \text{Hoyer}(\bm{y}) = \frac{\sqrt{d}-\|\bm{y}\|_1/\|\bm{y}\|_2}{\sqrt{d}-1}
\end{gathered}
\end{equation}
where $\Bar{\bm{z}}_i$ is a vector whose $d$-th dimensional value $\Bar{z}_{i,d}=z_{i,d}/\sigma(z_{i,d})$.
$\sigma(z_{i,d})$ is the standard deviation of $d$-th dimentional latent encoding taken over the dataset.
The $\text{Hoyer}(\bm{y}) \in [0,1]$ represents 0 for fully dense vector and 1 for a fully sparse vector.

We use the same convolutional neural networks for both the encoder and decoder as in \cite{mathieu2019disentangling} with $D$=50 dimensional latent space.
In this task, we use SGD optimizer with $C_1$=0.05, $\eta$=0.001, $\beta$=1, $B$=256, $b$=16, $L$=1 for all privatized models, and $C_2$=0.005 for \method with the MMD and $C_2$=0 for \method without it.
For non-private VAEs, we use Adam optimizer with $\eta$=0.0005, $B$=256.
For both VAE and \method, we set $\alpha$=100 when attaching the MMD.
We also compare with DP-SGD using micro-agg. and batch-agg..
For these methods, we set $C$=0.0002 to avoid exploding gradients.
The other hyper-parameters are the same as \method with MMD.
All models are trained in 10 epochs.

Figure \ref{fig:sparse_prior} shows the substantial sparsity by the sparse prior (Figure \ref{fig:fm_sparsity}), the log-likelihood (Figure \ref{fig:fm_recon}), and the MMD between q(z) and p(z) (Figure \ref{fig:fm_mmd}), those results are observed at several privacy parameter $\varepsilon$.
We plot the average over ten observations.
The shaded regions are $\pm$ 1 standard deviation around the averages.
In Figure \ref{fig:fm_sparsity}, \method with the regularization (\method+MMD) demonstrates higher sparsity than the model that does not have it.
Although it has a gap between the non-private regularized model (VAE+MMD), our proposed model achieved increasing the sparsity even under differential privacy constraints.
In the MMD between $q(z)$ and $p(z)$, \method+MMD shows smaller values against \method without it.
By employing the regularization term, \method could obtain the sparsity and reduce the MMD, but it was not easy to simultaneously increase the log-likelihood.
The trade-off between them seems more significant than non-private models.
To obtain more sparsity, \method needs to improve reconstruction performance.

%% file: 051_clustering.tex
\subsection{Clustering Latent Space}

Next, we consider a differentially private VAE that wishes to impose \textit{clustering} of the latent space.
For this experiment, we utilize the pinwheel dataset from \cite{johnson2016composing}, with $n$=400 observations, clustered in 4 spirals.
Following the experiment in \cite{mathieu2019disentangling}, we utilize a mixture of four Gaussians as the prior, $\text{KL}(p(z)||q(z))$ as a regularization divergence, and fully-connected neural networks for both encoder and decoder.
The prior is defined as 
$p(z)=\sum_{k=1}^K\pi^k \prod_{d=1}^D \mathcal{N}(\bm{z}|\mu^k_d,\sigma^k_d)$  
with $D$=2, $K$=4, $\sigma_d^k$=0.03, $\pi^k$=$1/K$, and $\mu_d^k\in\{0,1\}$.
The divergence is defined as
$\text{KL}(p(z)||q(z))\approx\sum_{j=1}^{|s|} (\log p(z_j)-\log \sum_{i=1}^{|s|}q(z_j|x_i) )$.
We set $C_1$=0.05, $\eta$=0.01, $B$=20, $b$=1, $L$=20 for all models, $C_2$=0.0005, $\beta$=0 for \method with $\text{KL}(p(z)||q(z))$ and $C_2$=0, $\beta$=1 for \method without it.

We compare the clustering performance between \method with/out the regularization term $\text{KL}(p(z)||q(z))$.
Figure \ref{fig:clustering} shows the reconstructions of the pinwheel data and the (clustered) representations.
The first two columns demonstrate the results of \method without $\text{KL}(p(z)||q(z))$, and the others show those of \method with $\text{KL}(p(z)||q(z))$.
In the figures the red dots represent the original inputs, the yellow dots are their reconstructions, and the blue dots show the data points in the latent spaces.
\method without the regularization demonstrates poor reconstructions against the raw pinwheel clustered data.
While, \method with $\text{KL}(p(z)||q(z))$ generated better reconstructions than the model without it even though the generated samples have small reconstruction errors.
The learned representations of \method with the regularization are well clustered and fitted to the prior that is the four mixture of Gaussians.
Through these results, our proposed model worked well with employing the prior and the regularization term those intended to capture the clusters of the pinwheel data.

\begin{figure}[t]
    \centering    
    \subfloat[\method ($\varepsilon=2.87$)]{
        \includegraphics[width=0.24\hsize]{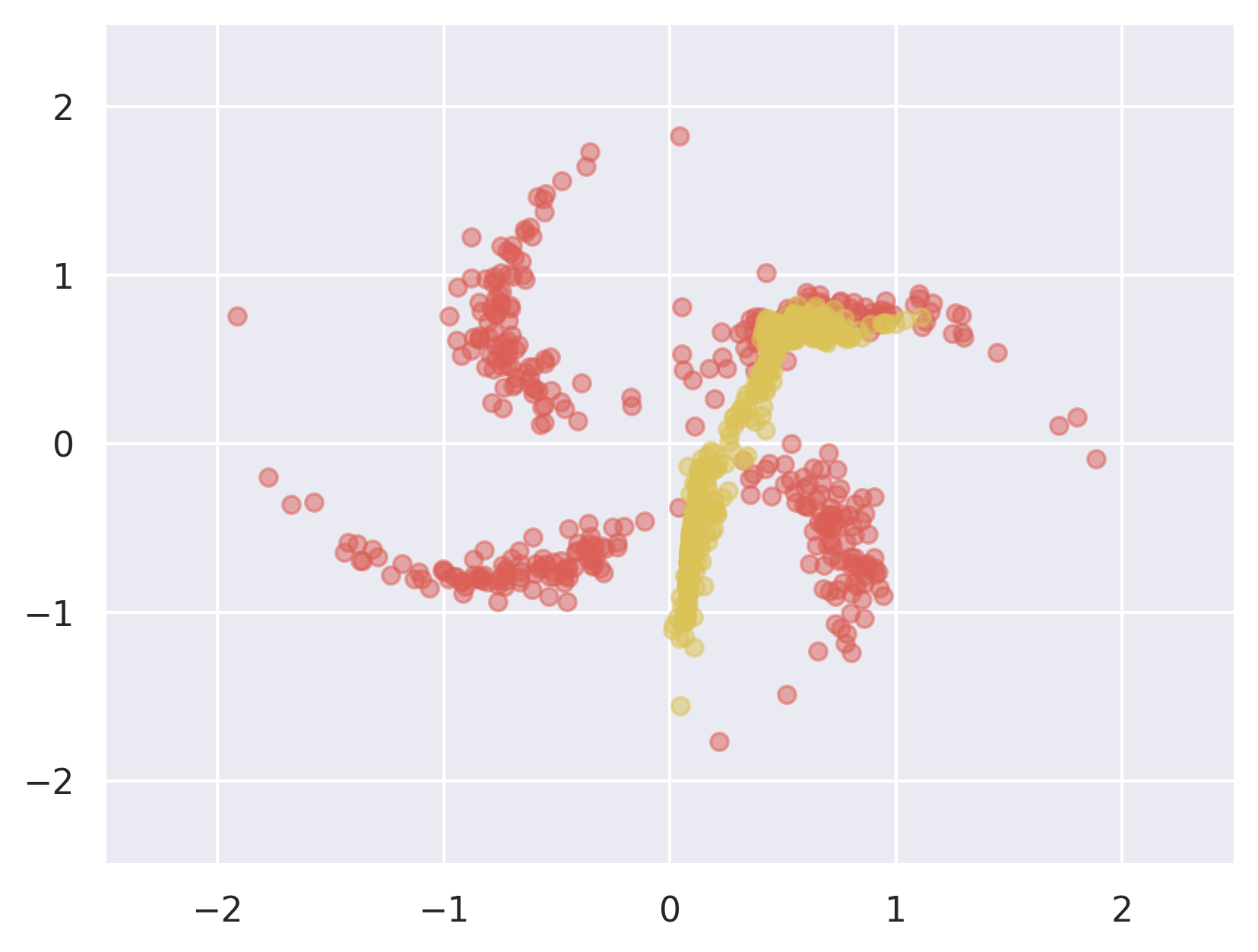}
        \includegraphics[width=0.24\hsize]{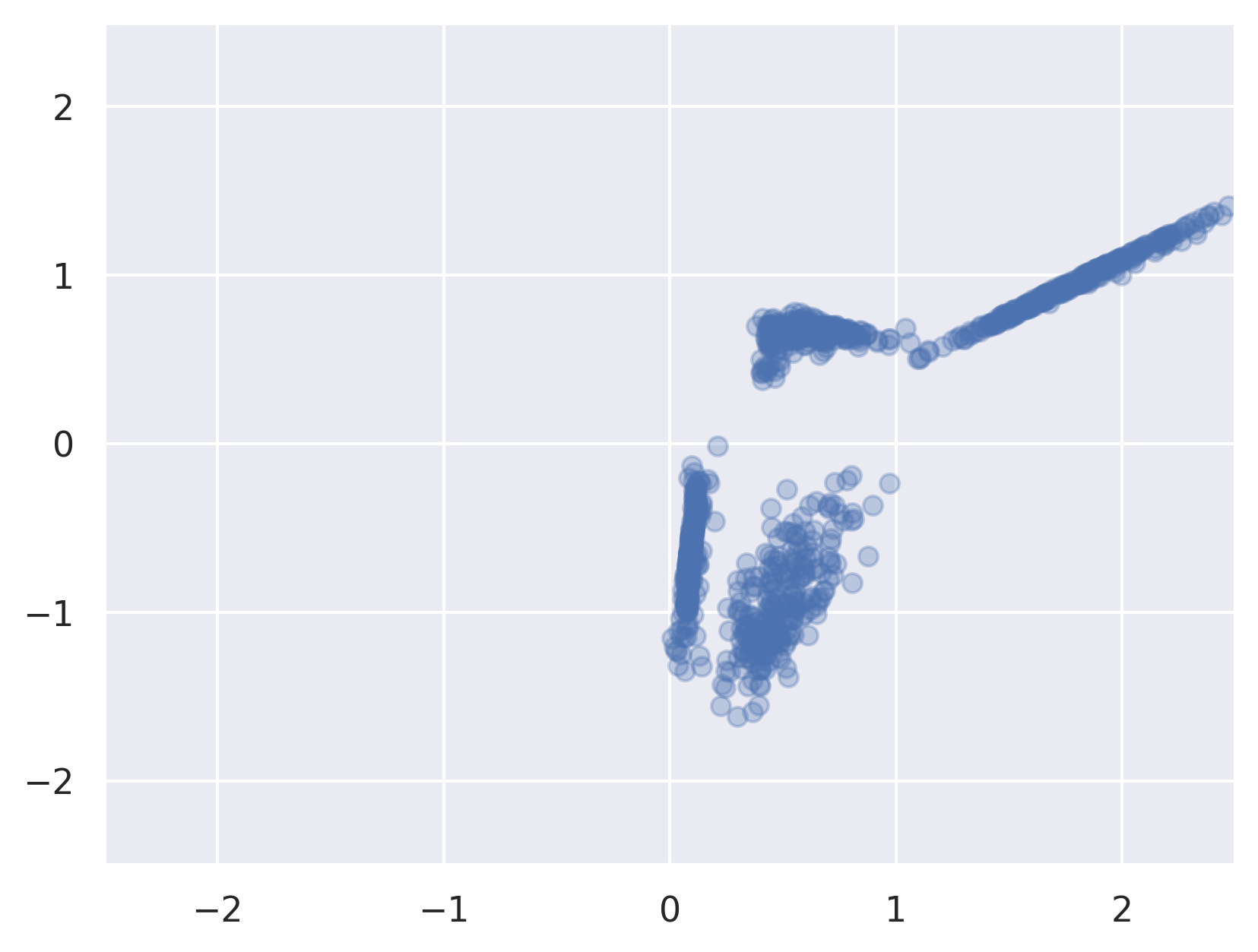}
        \label{fig:pin_rep_b1_1}
    }
    \subfloat[\method with $\text{KL}(p(z)||q(z))$ ($\varepsilon=2.87$)]{
        \includegraphics[width=0.24\hsize]{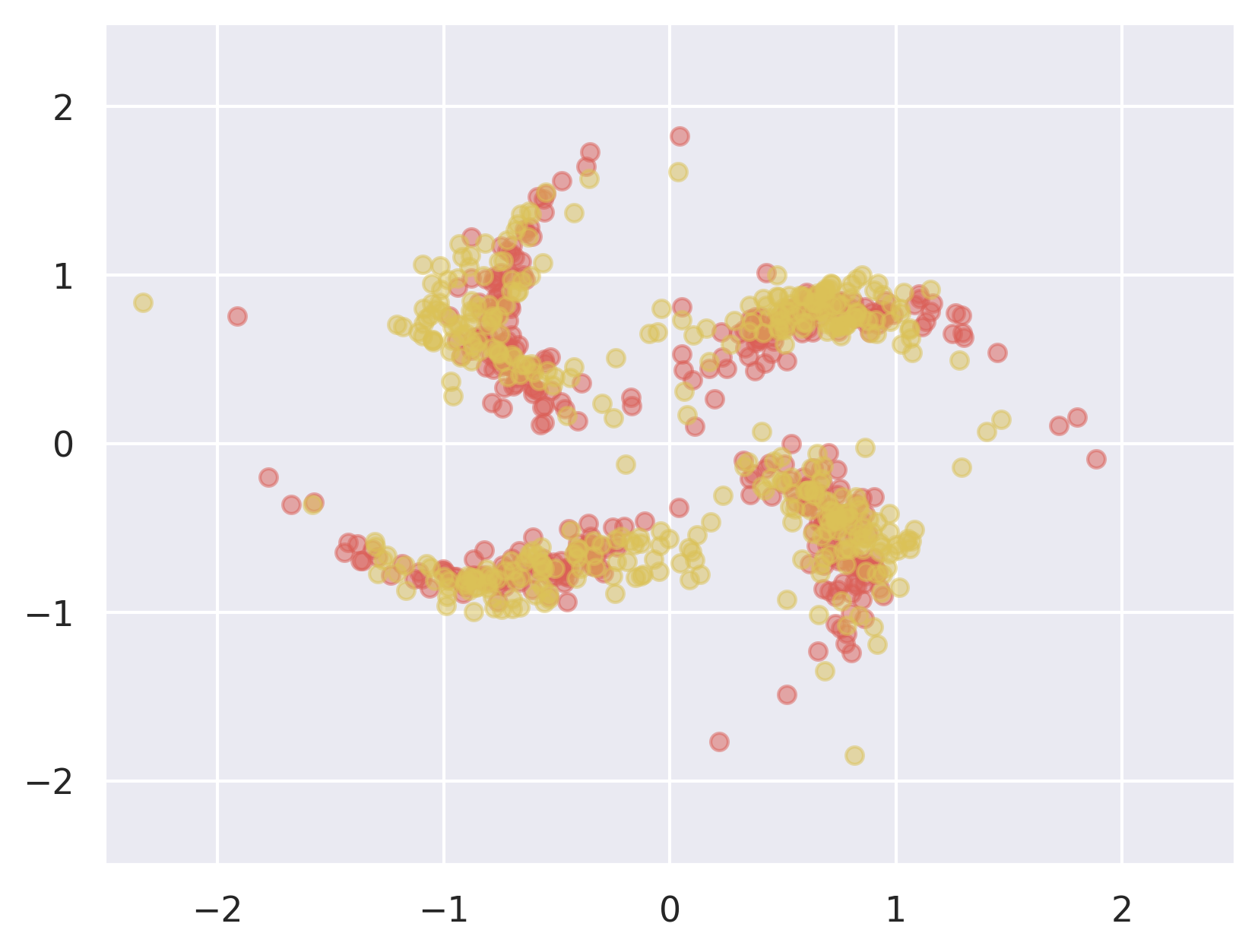}
        \includegraphics[width=0.24\hsize]{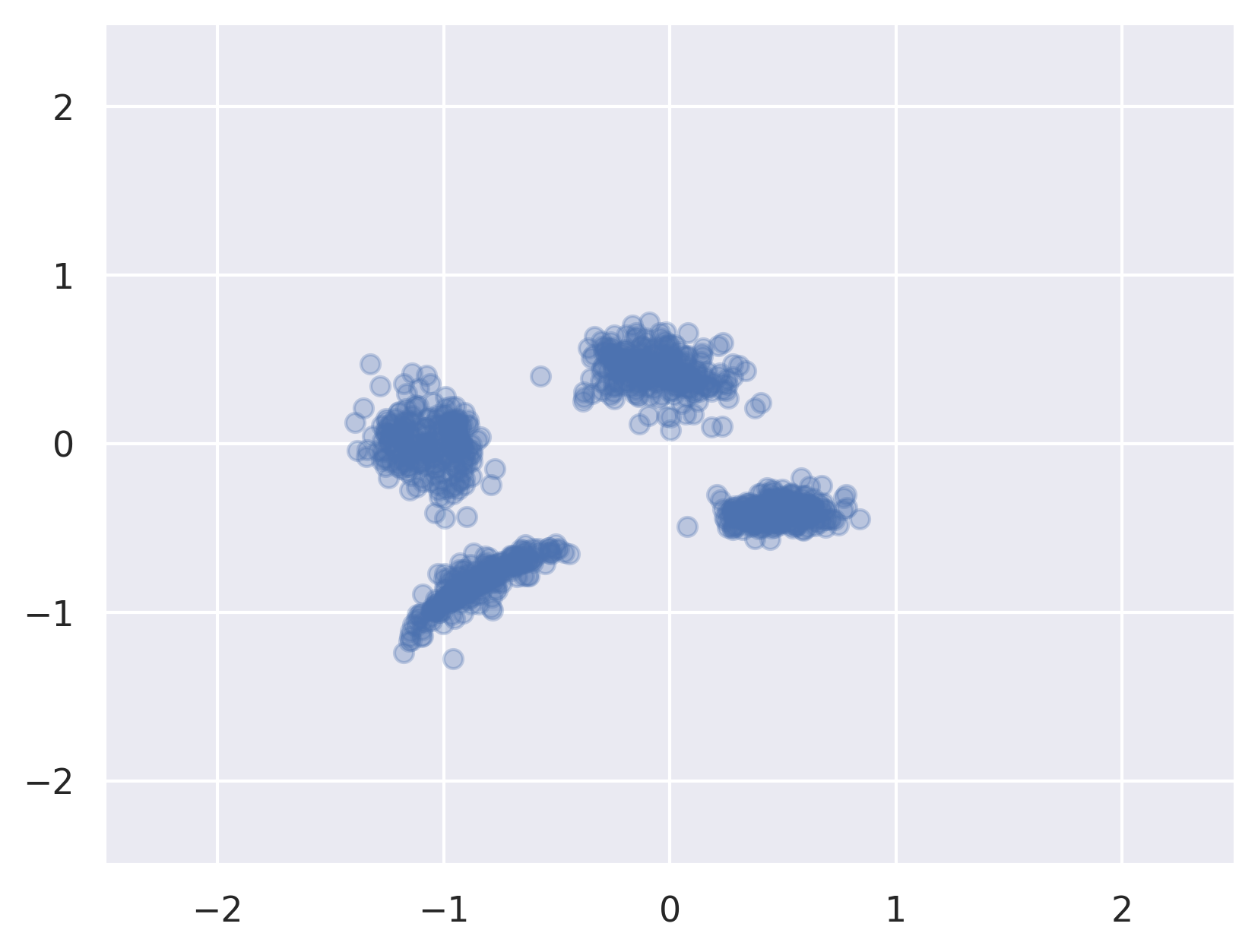}
        \label{fig:pin_rep_a1_1}
    }
    \caption{\method works for reconstructions with clustering representations. (a) and (b) show reconstructions (yellow) of the pinwheel data (red) and the learned representations (blue) for \method and \method with $\text{KL}(p(z)||q(z))$, respectively.
    \method with $\text{KL}(p(z)||q(z))$ demonstrates good reconstructions and representations for clustered pinweels, but \method without it shows poor results.}
    \label{fig:clustering}
\end{figure}

%% file: 060_conclusion.tex
\section{Conclusion}

This paper studied how to learn variational autoencoders with various divergence under differential privacy constraints.
We revealed several divergences increase the sensitivity of the stochastic gradient from $O(1)$ to $O(B)$ in terms of batch size $B$.
To reduce the sensitivity and the amount of noise, we proposed a term-wise DP-SGD that crafted randomized gradients in two different ways tailored to the compositions of the loss terms.
The term-wise DP-SGD could keep the sensitivity at $O(1)$ even when attaching the divergence.
In our experiments, we demonstrated that our method worked well with two pairs of the prior distribution and the divergence.
We mainly studied differentially private VAEs, but these contributions also have significant importance for the other machine learning models required to satisfy differential privacy.

%% file: ms.bbl
\begin{thebibliography}{10}

\bibitem{tensorflowprivacy}
Tensorflow privacy.
\newblock \url{https://github.com/tensorflow/privacy}.

\bibitem{abadi2016deep}
M.~Abadi, A.~Chu, I.~Goodfellow, H.~B. McMahan, I.~Mironov, K.~Talwar, and
  L.~Zhang.
\newblock Deep learning with differential privacy.
\newblock In {\em Proceedings of the 2016 ACM SIGSAC Conference on Computer and
  Communications Security}, pages 308--318. ACM, 2016.

\bibitem{abowd2018us}
J.~M. Abowd.
\newblock The us census bureau adopts differential privacy.
\newblock In {\em Proceedings of the 24th ACM SIGKDD International Conference
  on Knowledge Discovery \& Data Mining}, pages 2867--2867, 2018.

\bibitem{acs2018differentially}
G.~Acs, L.~Melis, C.~Castelluccia, and E.~De~Cristofaro.
\newblock Differentially private mixture of generative neural networks.
\newblock {\em IEEE Transactions on Knowledge and Data Engineering},
  31(6):1109--1121, 2018.

\bibitem{alemi2016deep}
A.~A. Alemi, I.~Fischer, J.~V. Dillon, and K.~Murphy.
\newblock Deep variational information bottleneck.
\newblock {\em arXiv preprint arXiv:1612.00410}, 2016.

\bibitem{amin2019differentially}
K.~Amin, T.~Dick, A.~Kulesza, A.~Munoz, and S.~Vassilvitskii.
\newblock Differentially private covariance estimation.
\newblock In {\em Advances in Neural Information Processing Systems}, pages
  14190--14199, 2019.

\bibitem{bagdasaryan2019differential}
E.~Bagdasaryan, O.~Poursaeed, and V.~Shmatikov.
\newblock Differential privacy has disparate impact on model accuracy.
\newblock In {\em Advances in Neural Information Processing Systems}, pages
  15453--15462, 2019.

\bibitem{bengio2013representation}
Y.~Bengio, A.~Courville, and P.~Vincent.
\newblock Representation learning: A review and new perspectives.
\newblock {\em IEEE transactions on pattern analysis and machine intelligence},
  35(8):1798--1828, 2013.

\bibitem{bindschaedler2016synthesizing}
V.~Bindschaedler and R.~Shokri.
\newblock Synthesizing plausible privacy-preserving location traces.
\newblock In {\em 2016 IEEE Symposium on Security and Privacy (SP)}, pages
  546--563. IEEE, 2016.

\bibitem{bindschaedler2017plausible}
V.~Bindschaedler, R.~Shokri, and C.~A. Gunter.
\newblock Plausible deniability for privacy-preserving data synthesis.
\newblock {\em Proceedings of the VLDB Endowment}, 10(5):481--492, 2017.

\bibitem{chaudhuri2019capacity}
K.~Chaudhuri, J.~Imola, and A.~Machanavajjhala.
\newblock Capacity bounded differential privacy.
\newblock In {\em Advances in Neural Information Processing Systems}, pages
  3469--3478, 2019.

\bibitem{chen2015differentially}
R.~Chen, Q.~Xiao, Y.~Zhang, and J.~Xu.
\newblock Differentially private high-dimensional data publication via
  sampling-based inference.
\newblock In {\em Proceedings of the 21th ACM SIGKDD International Conference
  on Knowledge Discovery and Data Mining}, pages 129--138. ACM, 2015.

\bibitem{dwork2006differential}
C.~Dwork.
\newblock Differential privacy.
\newblock In {\em Proceedings of the 33rd international conference on Automata,
  Languages and Programming-Volume Part II}, pages 1--12. Springer-Verlag,
  2006.

\bibitem{dwork2011differential}
C.~Dwork.
\newblock Differential privacy.
\newblock {\em Encyclopedia of Cryptography and Security}, pages 338--340,
  2011.

\bibitem{dwork2011firm}
C.~Dwork.
\newblock A firm foundation for private data analysis.
\newblock {\em Communications of the ACM}, 54(1):86--95, 2011.

\bibitem{eastwood2018framework}
C.~Eastwood and C.~K. Williams.
\newblock A framework for the quantitative evaluation of disentangled
  representations.
\newblock In {\em International Conference on Learning Representations}, 2018.

\bibitem{esmaeili2019structured}
B.~Esmaeili, H.~Wu, S.~Jain, A.~Bozkurt, N.~Siddharth, B.~Paige, D.~H. Brooks,
  J.~Dy, and J.-W. Meent.
\newblock Structured disentangled representations.
\newblock In {\em The 22nd International Conference on Artificial Intelligence
  and Statistics}, pages 2525--2534, 2019.

\bibitem{hurley2009comparing}
N.~Hurley and S.~Rickard.
\newblock Comparing measures of sparsity.
\newblock {\em IEEE Transactions on Information Theory}, 55(10):4723--4741,
  2009.

\bibitem{johnson2016composing}
M.~J. Johnson, D.~K. Duvenaud, A.~Wiltschko, R.~P. Adams, and S.~R. Datta.
\newblock Composing graphical models with neural networks for structured
  representations and fast inference.
\newblock In {\em Advances in neural information processing systems}, pages
  2946--2954, 2016.

\bibitem{jordon2018pate}
J.~Jordon, J.~Yoon, and M.~van~der Schaar.
\newblock Pate-gan: generating synthetic data with differential privacy
  guarantees.
\newblock 2018.

\bibitem{kingma2016improved}
D.~P. Kingma, T.~Salimans, R.~Jozefowicz, X.~Chen, I.~Sutskever, and
  M.~Welling.
\newblock Improved variational inference with inverse autoregressive flow.
\newblock In {\em Advances in neural information processing systems}, pages
  4743--4751, 2016.

\bibitem{kingma2013auto}
D.~P. Kingma and M.~Welling.
\newblock Auto-encoding variational bayes.
\newblock {\em arXiv preprint arXiv:1312.6114}, 2013.

\bibitem{lee2018concentrated}
J.~Lee and D.~Kifer.
\newblock Concentrated differentially private gradient descent with adaptive
  per-iteration privacy budget.
\newblock In {\em Proceedings of the 24th ACM SIGKDD International Conference
  on Knowledge Discovery \& Data Mining}, pages 1656--1665, 2018.

\bibitem{mathieu2019disentangling}
E.~Mathieu, T.~Rainforth, N.~Siddharth, and Y.~W. Teh.
\newblock Disentangling disentanglement in variational autoencoders.
\newblock In {\em International Conference on Machine Learning}, pages
  4402--4412, 2019.

\bibitem{mcmahan2018general}
H.~B. McMahan, G.~Andrew, U.~Erlingsson, S.~Chien, I.~Mironov, N.~Papernot, and
  P.~Kairouz.
\newblock A general approach to adding differential privacy to iterative
  training procedures.
\newblock {\em arXiv preprint arXiv:1812.06210}, 2018.

\bibitem{mcmahan2017learning}
H.~B. McMahan, D.~Ramage, K.~Talwar, and L.~Zhang.
\newblock Learning differentially private recurrent language models.
\newblock {\em arXiv preprint arXiv:1710.06963}, 2017.

\bibitem{papernot2016semi}
N.~Papernot, M.~Abadi, U.~Erlingsson, I.~Goodfellow, and K.~Talwar.
\newblock Semi-supervised knowledge transfer for deep learning from private
  training data.
\newblock {\em arXiv preprint arXiv:1610.05755}, 2016.

\bibitem{paszke2017automatic}
A.~Paszke, S.~Gross, S.~Chintala, G.~Chanan, E.~Yang, Z.~DeVito, Z.~Lin,
  A.~Desmaison, L.~Antiga, and A.~Lerer.
\newblock Automatic differentiation in pytorch.
\newblock 2017.

\bibitem{schein2019locally}
A.~Schein, Z.~S. Wu, A.~Schofield, M.~Zhou, and H.~Wallach.
\newblock Locally private bayesian inference for count models.
\newblock In {\em International Conference on Machine Learning}, pages
  5638--5648, 2019.

\bibitem{torkzadehmahani2019dp}
R.~Torkzadehmahani, P.~Kairouz, and B.~Paten.
\newblock Dp-cgan: Differentially private synthetic data and label generation.
\newblock In {\em Proceedings of the IEEE Conference on Computer Vision and
  Pattern Recognition Workshops}, pages 0--0, 2019.

\bibitem{wang2019sparse}
D.~Wang and J.~Xu.
\newblock On sparse linear regression in the local differential privacy model.
\newblock In {\em International Conference on Machine Learning}, pages
  6628--6637, 2019.

\bibitem{xiao2017fashion}
H.~Xiao, K.~Rasul, and R.~Vollgraf.
\newblock Fashion-mnist: a novel image dataset for benchmarking machine
  learning algorithms.
\newblock {\em arXiv preprint arXiv:1708.07747}, 2017.

\bibitem{xie2018differentially}
L.~Xie, K.~Lin, S.~Wang, F.~Wang, and J.~Zhou.
\newblock Differentially private generative adversarial network.
\newblock {\em arXiv preprint arXiv:1802.06739}, 2018.

\bibitem{yu2019differentially}
L.~Yu, L.~Liu, C.~Pu, M.~E. Gursoy, and S.~Truex.
\newblock Differentially private model publishing for deep learning.
\newblock In {\em 2019 IEEE Symposium on Security and Privacy (SP)}, pages
  332--349. IEEE, 2019.

\bibitem{zhang2014privbayes}
J.~Zhang, G.~Cormode, C.~M. Procopiuc, D.~Srivastava, and X.~Xiao.
\newblock Privbayes: private data release via bayesian networks.
\newblock In {\em Proceedings of the 2014 ACM SIGMOD International Conference
  on Management of Data}, pages 1423--1434, 2014.

\bibitem{zhang2016privtree}
J.~Zhang, X.~Xiao, and X.~Xie.
\newblock Privtree: A differentially private algorithm for hierarchical
  decompositions.
\newblock In {\em Proceedings of the 2016 International Conference on
  Management of Data}, pages 155--170. ACM, 2016.

\end{thebibliography}
